\documentclass[]{article}
\usepackage{amsmath}
\usepackage{amsfonts}
\usepackage{amssymb}
\usepackage{amsthm}
\usepackage{mathrsfs}
\usepackage{enumerate}
\usepackage[hidelinks]{hyperref} 
\usepackage{xspace}
\usepackage{xcolor} 

\allowdisplaybreaks[2]
\newtheorem{theorem}{Theorem}
\newtheorem{lemma}[theorem]{Lemma}

\newcommand{\eps}{\varepsilon}

\newcommand{\needle}{\textsc{Needle}\xspace}

\newcommand{\majority}{\mbox{$\textsc{Majority}_r$}\xspace}

\newcommand{\pim}{\ensuremath{p_i^-}\xspace}
\newcommand{\pip}{\ensuremath{p_i^+}\xspace}
\newcommand{\Tip}{\ensuremath{T_i^+}\xspace}

\newcommand{\R}{\ensuremath{\mathbb{R}}}

\newcommand{\N}{\ensuremath{\mathbb{N}}} 

\DeclareMathOperator{\expectation}{E}

\title{The Runtime of Random Local Search on the Generalized Needle Problem}
\author{Benjamin Doerr, Andrew James Kelley}

\begin{document}
	{\sloppy
		
		\maketitle
		
		\begin{abstract}
			In their recent work, C.~Doerr and Krejca (Transactions on Evolutionary Computation, 2023) proved upper bounds on the expected runtime of the randomized local search heuristic on generalized \needle functions. Based on these upper bounds, they deduce in a not fully rigorous manner a drastic influence of the needle radius~$k$ on the runtime.
			
			In this short article, we add the missing lower bound necessary to determine the influence of parameter~$k$ on the runtime. To this aim, we derive an exact description of the expected runtime, which also significantly improves the upper bound given by C.~Doerr and Krejca. We also describe asymptotic estimates of the expected runtime.
		\end{abstract}
		
		\section{Introduction}
		
		As part of a larger effort to design and analyze benchmarks with tunable features, C.~Doerr and Krejca~\cite{DoerrK23majority} propose a generalized \needle benchmark and analyze the runtime of the \emph{randomized local search} heuristic on this benchmark. 
		
		Defined on bit strings of length~$n$, the \needle problem with needle radius $k \in [0..n] := \{0, \dots, n\}$ consists of two plateaus of constant fitness. The global optimum consists of all bit strings having at most $k$ zeros. All other search points have the same lower fitness of zero. Hence, optimizing this benchmark  means finding a search point with at most $k$ zeros. Since the non-optimal solutions form a large plateau of constant fitness, this can still be difficult. The generalized \needle benchmark integrates into several attempts to understand how randomized search heuristics cope with plateaus of constant fitness such as studies on royal road functions~\cite{mitchell92royal,WegenerW05,DoerrSW13foga,DoerrK13cec}, the plateau problem~\cite{AntipovD21telo}, and the BlockLeadingOnes problem~\cite{DoerrK23gecco}.
		
		The main result~\cite[Theorem~7, but see further below for a discussion of the different notation]{DoerrK23majority} is that the \emph{randomized local search} heuristic (RLS) optimizes this benchmark, for $n$ even and sufficiently large, in an expected time of at most 
		\begin{equation}\label{eq:doerrk}
		6 r \, \frac{\lambda^r-1}{\lambda-1} + \tfrac 12 n(1+\ln r)
		\end{equation}
		iterations, where $r := \frac n2 - k$ and $\lambda = 1 + \frac{2n + 12r(r-1)}{3rn - 2n -6r(r-1)}$. 
		
		For this expression, the following asymptotic estimates were given. For $k \le (\frac 12 - \varepsilon)n$, $\varepsilon$ a positive constant, the expected runtime is at most ${(k+1) \exp(O(n^2 / (k+1))}$. When $k$ is so close to $\frac n2$ that the parameterization $r = \frac n2 - k$ satisfies $r = o(n)$, then for $r = \Omega(\sqrt{n \log\log n})$ the bound becomes $n \exp(r^2 / n)$. For all smaller~$r$, the runtime bound is $O(n \log r)$. From these estimates, a ``drastic change in the expected runtime'' (see the text following Theorem~7) is deduced in~\cite{DoerrK23majority}. While this is an intuitive claim, without lower bounds, of course, nothing can be said except that drastically changing upper bounds were proven.
		
		In this work, we provide the missing lower bounds by actually proving a precise expression for the runtime. If the initial search point has $i$ ones, then the expected runtime of RLS on the generalized \needle benchmark with parameter~$k$ is equal to
		\[\sum_{j=i}^{n-k-1} \binom{n}{\le j} \Big/ \binom{n-1}{j}.\]
		This result holds for all $n$, that is, the assumption that $n$ is even and sufficiently large is not required. Our proof is based on a classic Markov chain argument of Droste, Jansen, and Wegener~\cite{DrosteJW00}, and appears considerably simpler than the proof in~\cite{DoerrK23majority}, which uses various forms of drift, the optional stopping time theorem, and a generalized version of Wald's equation.
		
		We use this expression to give asymptotics of the expected runtime on $\needle_k$ in Theorems~\ref{thm:asymptotic_for_small_k}, \ref{thm:linear_k}, \ref{thm:bounds_for_k_near_half_n}, and \ref{thm:large_enough_k}. These results very roughly confirm the runtime behavior predicted (without proven lower bounds) in~\cite{DoerrK23majority},
		however, often our (asymptotically tight) estimates are considerably smaller than the upper bound given in~\cite{DoerrK23majority}. For example, for $k = o(n)$, which might be the most natural parameter range as only here the number of optima is sub-exponential, we show that the expected runtime is asymptotic to $2^n \binom{n}{k}^{-1}$, whereas the upper bound~\eqref{eq:doerrk} proven in~\cite{DoerrK23majority} in this case is superexponential. 
		
		We note that in~\cite{DoerrK23majority} a symmetric version of the \needle problem was also studied (called \textsc{HasMajority}). Here, each search point with at most $k$ ones {or} zeros is optimal. We note that our upper bounds also apply to this problem (naturally), and again significantly improve the results in~\cite{DoerrK23majority}. From the close similarity of the two variants of the problem, we are optimistic that our general method would allow one to determine the precise expected runtime of randomized local search on the symmetric \needle problem. Indeed, if $x_t \in \{0, 1\}^{n}$ is the individual at time $t$, one need only consider the Markov chain $X_t = \max\{\|x_t\|_1, n-\|x_t\|_1\}$. Since we do not expect much novelty from extending our analysis to the symmetric \needle problem, we do not further study this problem. 
		
		\section{Generalized Needle Functions}
		
		For $n \in \N$ and $k \in [0..n]$, let $\needle_{n,k}$ be the objective function defined by 
		\[\needle_{n,k}(x) =
		\begin{cases}
		0, &\text{if } \|x\|_1 < n-k,\\
		1, &\text{if } \|x\|_1 \ge n - k,
		\end{cases}\]
		for all $x \in \{0,1\}^n$. In other words, the global optimum of $\needle_{n,k}$ consists of the all-ones string and all bit strings that differ from it in at most $k$ bits. These solutions have an objective value of one. All other search points have the equal (and inferior) objective value of zero. This function class contains as special case $k=0$ the classic \needle benchmark (also called needle-in-a-haystack). 
		
		We call $n$ the \emph{problem size} and $k$ the \emph{needle radius}. As usual in the mathematical runtime analysis of randomized search heuristics~\cite{NeumannW10,AugerD11,Jansen13,ZhouYQ19,
			DoerrN20}, we often suppress the problem size~$n$ from our notation and write, e.g., just $\needle_k$. Also, we will often be interested in the asymptotic behavior of the runtime with respect to the problem size~$n$. In such cases, we use the common Landau symbols (``big-Oh notation''), always with respect to $n$ tending to infinity. We allow that the needle radius~$k$ be a function of the problem size~$n$, but again suppress this dependence from the notation, that is, write $k = \sqrt n$ instead of $k(n) = \sqrt n$. 
		
		This generalized \needle problem was introduced in~\cite{DoerrK23majority} (for even $n$ and $k \in [0..\frac n2]$) under the name \emph{generalized majority problem}. Since this name hides the relation with the classic \needle problem and since a different benchmark named \emph{majority} was defined~\cite{GoldbergO98} and studied~\cite{DurrettNO11,Neumann12,NguyenUW13,KotzingLLM20,DoerrKLL20} earlier, we suggest to use the name \emph{generalized \needle problem}. To ease the comparison with~\cite{DoerrK23majority}, we note that \[\needle_{k} = \majority\] 
		with $r=\frac n2 - k$ for all $k \in [0..\frac 12n]$.

		\section{Exact Runtimes}
		
		In this section, we prove our result that determines the exact runtime of randomized local search on the generalized \needle problem.
		
		\begin{theorem}\label{thm:exact_expectation}
			Let $n \in \N$ and $k \in [0..n]$. Let $i \in [0..n]$. Let $T(i)$ be the runtime of RLS on $\needle_{n,k}$ when starting with an initial solution having exactly $i$ ones. Then 
			\[
			\expectation[T(i)] = \sum_{j=i}^{n-k-1} \binom{n}{\leq j} \Big/ \binom{n-1}{j}
			\]
			for $i \le n-k-1$ and $\expectation[T(i)]=0$ for $i \in [n-k..n]$.
			
			Let $T$ be the runtime when starting with a random solution. Then
			\[
			\expectation[T] 
			= \sum_{i=0}^{n-k-1} \left[\binom{n}{i} 2^{-n} \sum_{j=i}^{n-k-1} \binom{n}{\leq j} \Big/ \binom{n-1}{j}  \right]. 
			\]
		\end{theorem}
		
		When running randomized local search on a generalized \needle problem, in each iteration until an optimum is found, the new solution is accepted. Consequently, in this phase, randomized local search performs a classic unbiased random walk on the hypercube $\{0,1\}^n$. This changes once an optimum is reached, but since we are only interested in the first hitting time of an optimal solution, for our analysis it suffices to regard an unbiased random walk on the hypercube. 
		
		It is easy to see that for our runtime question, it suffices to regard the reduced Markov chain on $[0..n]$ that only regards the number of ones in the current state of the random walk on $\{0,1\}^n$. This argument has been used already in the first runtime analysis of the \needle problem~\cite{GarnierKS99}, the formal reason behind this intuitive argument is that both a run of randomized local search on a generalized \needle problem and the set of optima of this problem are invariant under coordinate permutations (see~\cite{Doerr21symmetry} for a more formal discussion of this argument). 
		
		We thus now analyse this Markov chain on $[0..n]$. We note that for $i \in [1..n]$, the transition probability to go from $i$ to $i-1$, is
		\[
		\pim := \tfrac in.
		\]
		For $i \in [0..n-1]$, the transition probability to go from $i$ to $i+1$ is
		\[
		\pip:=\tfrac{n-i}{n}.
		\]
		
		For such Markov chains, that is, Markov chains defined in an interval of the integers that move at most to neighboring states, hitting times can be computed exactly from the transition probabilities. The following formula is part of Corollary 5 from \cite{DrosteJW00}.
		\begin{lemma}
			\label{lem:droste}
			Let \Tip denote the (first) hitting time to reach state $i+1$, given that the Markov chain starts in state $i$. Then
			\[
			\expectation[\Tip] = \sum_{k=0}^i \frac{1}{p_k^+} \prod_{\ell = k+1}^i \frac{p_\ell^-}{p_\ell^+}.
			\]
		\end{lemma}
		
		We shall occasionally need the statement that $\expectation[T_i^+]$ is increasing in~$i$. Given that progress is harder when closer to the optimum, this statement is very natural. The proof, while not difficult, is slightly less obvious.
		
		\begin{lemma}
			\label{lem:Tip_is_increasing}
			Let \Tip denote the (first) hitting time to reach state $i+1$, given that the Markov chain starts in state $i$. Then $\expectation[\Tip]$ is increasing in $i$ for $i \in [0..n]$.
		\end{lemma}
		\begin{proof}
			Let $0\leq i < n$. 
			Lemma~\ref{lem:droste} gives
			\[
			\expectation[\Tip] = \sum_{k=0}^i \frac{1}{p_k^+} \prod_{\ell = k+1}^i \frac{p_\ell^-}{p_\ell^+},
			\]
			and 
			\[
			\expectation[T_{i+1}^+] = \sum_{k=0}^{i+1} \frac{1}{p_k^+} \prod_{\ell = k+1}^{i+1} \frac{p_\ell^-}{p_\ell^+}.
			\]
			We will be done once we show that the $k$-th term in the sum for $\expectation[\Tip]$ is less than the $(k+1)$-st term of $\expectation[T_{i+1}^+]$. In other words, we want
			\[
			\frac{1}{p_k^+} \prod_{\ell = k+1}^i \frac{p_\ell^-}{p_\ell^+} < \frac{1}{p_{k+1}^+} \prod_{\ell = k+2}^{i+1} \frac{p_\ell^-}{p_\ell^+}.
			\]
			This follows from the facts that $p_j^+$ is decreasing and $p_j^-$ is increasing in~$j$.
		\end{proof}
		
		For our Markov chain, this result gives the following simple expression for the expected time to increase the current state.
		\begin{lemma}
			\label{lem:droste_simplified}
			Consider still the Markov chain that counts the number of ones of the state of a random walk on the hypercube $\{0,1\}^n$. Let \Tip denote the (first) hitting time to reach state $i+1$, given that the Markov chain starts in state $i$.  We have
			\[
			\expectation[\Tip] = \binom{n}{\leq i} \Big/ \binom{n-1}{i}
			\]
			for all $i \in [0..n-1]$.
		\end{lemma}
		
		\begin{proof}
			With $\pim=i/n$ and $\pip = (n-i)/n$, Lemma~\ref{lem:droste} gives
			\[
			\begin{aligned}
			\expectation[\Tip] &= \sum_{k=0}^i \frac{n}{n-k} \prod_{\ell = k+1}^i \frac{\ell/n}{(n-\ell)/n} \\
			&= \sum_{k=0}^i \frac{n}{n-k} \prod_{\ell = k+1}^i \frac{\ell}{n - \ell} \\
			&= \sum_{k=0}^i n \cdot \frac{(k+1)(k+2)\cdots i}{(n-k)(n-(k+1))\cdots(n-i)} \\
			&= \sum_{k=0}^i n \cdot \frac{i!(n-i-1)!}{k!(n-k)!} \\
			&=  \sum_{k=0}^i n \cdot \frac{(n-1)! \binom{n-1}{i}^{-1}}{n!\binom{n}{k}^{-1}} \\
			&= \frac{1}{\binom{n-1}{i}} \sum_{k=0}^i \binom{n}{k}. 
			\end{aligned}
			\]
		\end{proof}
		
		Now, standard Markov chain arguments easily prove Theorem~\ref{thm:exact_expectation}.
		\begin{proof}[Proof of Theorem~\ref{thm:exact_expectation}]
			From basic properties of Markov chains and the connection made earlier between random walks on the hypercube and the runtimes we are interested in, we see that the time $T(i)$ to reach an optimum when starting with $i \le n-k-1$ ones satisfies $\expectation[T(i)] = \sum_{j=i}^{n - k -1} \expectation[T_j^+]$. Thus Lemma~\ref{lem:droste_simplified} gives
			\[
			\expectation[T(i)] = \sum_{j=i}^{n-k-1} \binom{n}{\leq j} \Big/ \binom{n-1}{j}.
			\]
			Trivially, $T(i)=0$ when $i \ge n-k$. 
			
			Let $X$ denote the number of ones of the random initial solution of a run of RLS on $\needle_k$, and note that $X$ follows a binomial law with parameters $n$ and $p=1/2$. By the law of total expectation, the runtime $T$ of this run satisfies
			\begin{align*}
			\expectation[T] 
			&= \sum_{i=0}^{n} \Pr[X = i] \, \expectation[T(i)] \\
			&= \sum_{i=0}^{n} \left[\binom{n}{i} 2^{-n} \expectation[T(i)]  \right] \\
			&= \sum_{i=0}^{n-k-1} \left[\binom{n}{i} 2^{-n} \sum_{j=i}^{n-k-1} \binom{n}{\leq j} \Big/ \binom{n-1}{j}  \right].  
			\end{align*}
		\end{proof}
		
		\section{Estimates for the Runtime}
		
		In Theorem~\ref{thm:exact_expectation}, we determined the expected runtime of RLS on generalized \needle functions. Since the precise expression for this runtime is not always easy to interpret, we now estimate this expression. 
		
		We start with the most natural case that $k = o(n)$, which is equivalent to saying that we have a sub-exponential number of global optima, and then move on to larger values of~$k$.
		
		\subsection{Sub-linear Values of $k$}
		
		\begin{theorem}
			\label{thm:asymptotic_for_small_k}
			Let $k = o(n)$ (with $k$ possibly constant).
			Let $T$ be the time spent to optimize $\needle_{n, k}$. Then
			\[
			\expectation[T] \sim 2^n \binom{n}{k}^{-1}.
			\]
			This estimate is of asymptotic order $2^{(1-o(1)) n}$ for all values of~$k$.
		\end{theorem}
		
		To prove this result, we need the following estimate for $T(i)$, which might be of independent interest.
		
		\begin{lemma}
			\label{lem:Ti_for_k_o_of_n}
			Let $k = o(n)$, and let $T(i) = T_n(i)$ be the runtime of RLS on $\needle_{n,k}$ when starting with an initial solution having exactly $i$ ones. Then for all $i \leq n - k -1$, as $n \to \infty$,
			\[
			\expectation[T_n(i)] \sim 2^n \binom{n}{k}^{-1}.
			\]
		\end{lemma}
		\begin{proof}
			Let $X$ be the number of ones of the random initial solution. So $X$ is a binomial random variable with parameters $n$ and $1/2$. Then $\Pr[X \le n-k-1] \sim 1$ because $\Pr[X \le n/2 + r] \sim 1$  for all $r = \omega(\sqrt n)$ by well-known properties of the binomial distribution (or via an elementary application of an additive Chernoff bound, e.g., Theorem~1.10.7 in~\cite{Doerr20bookchapter}). Therefore, $\binom{n}{\leq n-k-1} \sim 2^n$. Consequently, we will be finished when we show
			\[
			\expectation[T(i)] \sim \binom{n}{\leq n-k-1} \binom{n}{k}^{-1}.
			\]
			
			By Theorem~\ref{thm:exact_expectation}, the runtime of RLS on $\needle_k$, starting from $i \le n-k-1$ ones, is
			\[\tag{1}\label{eq:expectation_Ti}
			\expectation[T(i)] = \sum_{j=i}^{n-k-1} \binom{n}{\leq j} \Big/ \binom{n-1}{j}.
			\]
			Taking only the last term in this sum, we see that 
			\[
			\begin{aligned}
			\expectation[T(i)] \geq \binom{n}{\leq n-k-1} \Big/ \binom{n-1}{n-k-1} &= \binom{n}{\leq n-k-1} \Big/\binom{n-1}{k} \\
			&\geq \binom{n}{\leq n-k-1} \Big/\binom{n}{k}.
			\end{aligned}
			\]
			
			Next, we show that $\expectation[T(i)] \leq (1+o(1)) \binom{n}{\le n-k-1} / \binom{n}{k}$, which will complete this proof. Since $\expectation[T(i)]$ is greatest for $i = 0$, we need only show the upper bound for $\expectation[T(0)]$. 
			
			To this aim, let us consider a run of RLS on $\needle_{n,k}$ started in the all-zeros solution. Let $I_t$ be the number of ones of the current individual in generation $t$, and let $Y$ be the number of steps until the current individual has $n/4$ ones (i.e.\ $Y$ is the minimum of $t$ such that $I_t = n/4$). While $I_t \leq n/4$, notice that we have $\Pr(I_{t+1} = I_t + 1) \geq 3/4$, and $\Pr(I_{t+1}) = I_t - 1 \leq 1/4$. Consequently, the expected drift of $I_t$ toward $n/4$ is 1/2 (in other words, for $s \leq n/4$, we have $\expectation[I_{t+1}-I_t \mid I_t = s] \geq 1/2$.) Therefore, by the additive drift theorem of He and Yao~\cite{HeY01} (Theorem~1 of~\cite{Lengler20bookchapter}), we have $\expectation[Y]  \leq (n/4)/(1/2) = n/2$. Therefore, 
			$$\expectation[T(0)] \leq n/2 + \expectation[T(n/4)].$$
			
			Because $n/2 = o\big(\binom{n}{\le n-k-1} / \binom{n}{k}\big)$, the conclusion of the previous paragraph shows that we need only show that $\expectation[T(n/4)] \leq (1+o(1)) \binom{n}{\le n-k-1} / \binom{n}{k}$.
			
			Let $S(n) = \sum_{j=n/4}^{n-k-1} \binom{n-1}{j}^{-1} = \sum_{j=n/2-1}^{n-k-1}\binom{n-1}{j}^{-1} + \sum_{j=n/4}^{n/2}\binom{n-1}{j}^{-1}$. For the first summation, replace $\binom{n-1}{j}^{-1}$ with $\binom{n-1}{n-1- j}^{-1}$, 
			the biggest term then is $\binom{n-1}{k}^{-1}$. We then have $S(n) = \sum_{j=k}^{n/2}\binom{n-1}{j}^{-1} + \sum_{j=n/4}^{n/2}\binom{n-1}{j}^{-1}$.
			Because $k = o(n)$, that implies that $(1+o(1)) \sum_{j=k+1}^{n/2}\binom{n-1}{j}^{-1} > \sum_{j=n/4}^{n/2}\binom{n-1}{j}^{-1}$, and hence
			\[\tag{2}\label{Sn_upper_bound}
			S(n) \leq \binom{n-1}{k}^{-1} + 2 (1+o(1)) \cdot \sum_{j=k+1}^{n/2}\binom{n-1}{j}^{-1}.
			\]
			
			Let $x_j$ be such that $\binom{n-1}{j}^{-1}x_j = \binom{n-1}{j+1}^{-1}$. Then $x_j = (j+1)/(n-j-1)$. 
			Notice that $x_k = o(n)/(n-o(n)) = o(1)$. Let $m = n/4 - 1$. Then we have that $x_m = (n/4)/(n-n/4) = 1/3$, and hence for $j \leq n/4$, we have $x_j \leq 1/3$. Therefore,
			\[
			\begin{aligned}
			\sum_{j=k+1}^{n/2}\binom{n-1}{j}^{-1} &= \sum_{j=k+1}^{n/4-1} \binom{n-1}{j}^{-1} + \sum_{j=n/4}^{n/2} \binom{n-1}{j}^{-1} \\
			&\leq \sum_{j=k+1}^{n/4-1} (1/3)^{j - k - 1}\binom{n-1}{k+1}^{-1} + \sum_{j=n/4}^{n/2} (1/3)^{\Theta(n)}\binom{n-1}{k+1}^{-1} \\
			&= \Theta\left( \binom{n-1}{k+1}^{-1}\right).
			\end{aligned}
			\]
			
			Combining this with inequality (\ref{Sn_upper_bound}) on $S(n)$ yields $S(n) \leq \binom{n-1}{k}^{-1} + \Theta\left( \binom{n-1}{k+1}^{-1}\right)$. Earlier, we noted that for $\binom{n-1}{k}^{-1}x_k = \binom{n-1}{k+1}^{-1}$, we have $x_k = o(1)$. We may conclude then that $S(n) \leq (1+o(1))\binom{n-1}{k}^{-1}$. From equality (\ref{eq:expectation_Ti}) at the beginning of the proof and the definition of $S(n)$, we get 
			\[
			\begin{aligned}
			\expectation[T(n/4)] &\leq \binom{n}{\leq n - k -1} S(n) \\
			&\leq (1+o(1))\binom{n}{\leq n - k -1}\binom{n-1}{k}^{-1}\\
			&= (1+o(1))\binom{n}{\leq n - k -1}\binom{n}{k}^{-1}, 
			\end{aligned}
			\]
			where the equality absorbed into the $o(1)$ the term $\binom{n-1}{k}^{-1}/\binom{n}{k}^{-1} \sim 1$. We are done once we recall that we already showed $\expectation[T(0)] \leq n/2 + \expectation[T(n/4)]$.
		\end{proof}
		
		\begin{proof}[Proof of Theorem~\ref{thm:asymptotic_for_small_k}]
			Let $\expectation[T(i)]$ be the runtime of RLS on $\needle_{n,k}$ when starting with an initial solution having exactly $i$ ones.
			By Lemma~\ref{lem:Ti_for_k_o_of_n}, for all $i \leq n - k -1$, as $n \to \infty$,
			\[
			\expectation[T(i)] \sim 2^n \binom{n}{k}^{-1}.
			\]
			We have $\expectation[T] = \sum_{i=0}^n \Pr[X = i]\expectation[T(i)] = \sum_{i=0}^{n-k-1} \Pr[X = i]\expectation[T(i)]$. To get an upper bound, replace each $\expectation[T(i)]$ with $\expectation[T(0)]$, and to get a lower bound, replace each $\expectation[T(i)]$ with $T(n-k-1)$. Recall $\Pr[X \le n/2 + r] \to 1$ for all $r = \omega(\sqrt n)$ by well-known properties of the binomial distribution (or via an elementary application of an additive Chernoff bound, e.g., Theorem~1.10.7 in~\cite{Doerr20bookchapter}). 
			Hence, because $\Pr[X \leq n - k -1] \to 1$, we get 
			\[
			\sum_{i=0}^{n-k-1} \Pr[X = i]\expectation[T(i)] \sim 2^n \binom{n}{k}^{-1}.
			\]
			That this last expression is $2^{(1-o(1))n}$ follows easily from the standard estimate $\binom{n}{k} \le (\frac {ne}k)^k$ and noting that $(\frac nk)^k = \exp(o(n))$ for $k=o(n)$.
		\end{proof}
		
		\subsection{Linear Values for $k$, Linearly Bounded Away From~$n/2$}
		
		In the case of sublinear values of $k$ just analyzed, we conveniently exploited the fact that that we had a precise estimate for $T(i)$ that was independent of~$i$ (Lemma~\ref{lem:Ti_for_k_o_of_n}).
		
		This approach is not possible once we turn to the case that $k$ is linear in $n$, but less than $n/2$, since the distance to $n/2$ will be the crucial quantity. We use the parameterization $k = n/2 - \eps n$ for a given positive $\eps$ that is less than $1/2$. 
		
		In this case, the expectations of the $T(i)$, as determined in Theorem~\ref{thm:exact_expectation}, will not be identical apart from lower order terms. Fortunately, we will show that there is a starting value $a$ such that the expected runtime $\expectation[T]$ is of the same asymptotic order as the expected time $\expectation[T(a)]$ when starting with $a$ ones. This statement is made precise in the following lemma, from which we will then easily derive our runtime estimate in Theorem~\ref{thm:linear_k}.
		
		\begin{lemma}
			\label{lem:expectation_is_Theta_expectation_Ta}
			Let $k \leq n/2$ and let $T$ be the runtime on $\needle_{k}$. Let $X$ be a binomial random variable with parameters $n$ and 1/2. Let 
			\[
			a = \big\lfloor \expectation[X \mid X \le n-k]\,\big\rfloor.
			\]
			Then $\expectation[T] = \Theta(\expectation[T(a)])$.
		\end{lemma}
		
		The proof of this result, building on Jensen's inequality for convex/concave functions, relies on several technical lemmas, which we provide now. The first of these will ensure that $T(i)$, viewed as a function in $i$ and linearly interpolated between its integral arguments, is a concave function.
		
		\begin{lemma}
			\label{lem:concavity_from_n_points}
			Let $f : [0.. k-1] \to \R$ be a nondecreasing function. Also, define $\tilde{F} : [0..k] \to \R$  
			by $\tilde{F}(i) = \sum_{j=i}^{k-1} f(j)$ and $\tilde{F}(k) = 0$, and define $F: [0, k] \to \R$ from $\tilde{F}$ by linear interpolating between
			the points. In other words, for $x \in [0,k]$,
			\[
			F(x) =
			\begin{cases}
			\tilde{F}(x) &\text{ if } x \in [0..k] \\
			p\tilde{F}(\lfloor x \rfloor) + (1-p) \tilde{F}(\lceil x \rceil) &\text{ if } x = p\lfloor x \rfloor + (1-p)\lceil x \rceil  \\
			& \text{\quad\quad\quad for some  } p \in (0, 1).
			\end{cases}
			\]
			Then $F$ is a concave function on $[0, k]$.
		\end{lemma}
		\begin{proof}
			A continuous function that is piecewise linear is concave if and only if
			the sequence of successive slopes is nonincreasing. For $1 \leq i \leq k-1$, the $i$th slope of $F$
			is 
			\[
			\frac{F(i) - F(i-1)}{i - (i-1)} = \sum_{j=i}^{k-1} f(j) - \sum_{j=i-1}^{k-1} f(j)
			= -f(i-1).
			\]
			The $k$th slope is $-f(k-1)$. Since $f$ is nondecreasing, $-f$ is nonincreasing, and this proves the claim.
		\end{proof}
		
		With the above lemma, we can show the following claim, which will then, with a suitable estimate of the quantity~$u$, imply Lemma~\ref{lem:expectation_is_Theta_expectation_Ta}.
		
		\begin{lemma}	\label{lem:theta_bound_on_expectation_T}
			Let $T$ be the runtime on $\needle_k$, and let
			$T(j)$ be the runtime when starting from an individual with $j$ ones. Let $X$ be a binomial random variable with parameters $n$ and 1/2. For all $i \in [0..n]$, let $p_i = \Pr[X = i]$. Let further
			\begin{align*}
			w &= \sum_{i=0}^{n - k}p_i = \Pr[X \le n-k],\\
			a &= \left\lfloor w^{-1} \sum_{i=0}^{n - k}p_i i \right\rfloor = \big\lfloor \expectation[X \mid X \le n-k]\, \big\rfloor,\\
			u &= \sum_{i=0}^{a}p_i = \Pr[X \le a].     
			\end{align*}
			Then
			\[
			u \expectation[T(a)] \leq \expectation[T] \leq w \expectation[T(a)].
			\]
		\end{lemma}
		\begin{proof}
			By Lemma~\ref{lem:Tip_is_increasing}, $\expectation[\Tip]$ is increasing in $i$.
			Therefore, Lemma~\ref{lem:concavity_from_n_points} implies that $\expectation[T(i)]$ (with
			linear interpolation) is a concave function on $[0,n-k]$. Jensen's inequality then
			implies that 
			\[
			\frac{\expectation[T]}{w} \leq \expectation[T(b)],
			\]
			where $b = w^{-1} \sum_{i=0}^{n-k}p_i i$. Then $\expectation[T] \leq w \expectation[T(a)]$ since
			$a = \lfloor b \rfloor$ and since $\expectation[T(i)]$, and thus also its linear interpolation, are nonincreasing.
			
			The lower bound follows since with probability $u$, the starting point of the algorithm is with an individual having $a$ ones or fewer.
		\end{proof}
		
		What is missing for the proof of Lemma~\ref{lem:expectation_is_Theta_expectation_Ta} is the following estimate on the conditional expectation of a binomial random variable.
		
		\begin{lemma}
			\label{lem:binom_expectation_ineq}
			Let $h : \N \to \R_{\ge 0}$. 
			Let $X$ be a binomial random variable with parameters $n$ and 1/2. Let $A$ be the event that $X \leq n/2 + h(n)$. Then 
			\[
			\expectation[X \mid A] \geq n/2 - \sqrt{\frac{n}{2\pi}} - o(\sqrt{n}\,).
			\]
		\end{lemma}
		
		\begin{proof}
			Clearly, $\expectation[X \mid A]$ increases if we replace $h(n)$ with any $H(n) \geq h(n)$, and so it suffices to prove the claim for $h(n)=0$.
			
			First, assume that $n$ is even. Let $S = \sum_{i=0}^{n/2 -1} \binom{n}{i}$. Since $2^n = \sum_{i=0}^{n} \binom{n}{i}$, and $\binom{n}{i} = \binom{n}{n-i}$ for all $i \in [0..n]$, we have $S + \binom{n}{n/2} = \big(2^n + \binom{n}{n/2}\big)/2$. Therefore, 
			
			\[\begin{aligned}
			\Pr(A) &= 2^{-n}\left(S + \binom{n}{n/2}\right) \\
			&= (2^{-n}) \frac{2^n + \binom{n}{n/2}}{2}\\
			&= \frac{1 + 2^{-n}\binom{n}{n/2}}{2}\\
			&= \frac{1 + (1+o(1))\sqrt{2/(\pi  n)}}{2},
			\end{aligned} 
			\]
			where in this last step, we used that $\binom{n}{n/2} = (1+o(1)) 2^n \sqrt{2/(\pi  n)}$.
			
			We have that
			\[\begin{aligned}
			\expectation[X \mid A] &= \frac{1}{\Pr(A)}\sum_{k=0}^{n/2} k \binom{n}{k}2^{-n}  \\
			&= \frac{1}{2^n\Pr(A)}\sum_{k=1}^{n/2} k \binom{n}{k} \\
			&= \frac{1}{2^n\Pr(A)}\sum_{k=1}^{n/2} n \binom{n-1}{k-1} \\
			&= \frac{n}{2^n\Pr(A)}\sum_{j=0}^{n/2-1} \binom{n-1}{j}\\
			&= \frac{n}{2^n\Pr(A)} \cdot \frac{2^{n-1}}{2}\\
			&= \frac{n}{4\Pr(A)}
			\end{aligned}
			\]
			Using what we found $\Pr(A)$ to be and also using $1/(1+x) \geq 1 - x$ gives
			\[\begin{aligned}
			\expectation[X \mid A] &= \frac{n}{2}\left(\frac{1}{1 + (1 + o(1))\sqrt{2/(\pi  n)}}\right)\\
			&\geq \frac{n}{2}\left((1 - (1 + o(1))\sqrt{2/(\pi  n)}\right),
			\end{aligned}
			\]
			which proves this result for even $n$. A similar argument works for odd $n$.
		\end{proof}

		We are now ready to prove Lemma~\ref{lem:expectation_is_Theta_expectation_Ta} as an easy consequence of Lemmas~\ref{lem:theta_bound_on_expectation_T} and \ref{lem:binom_expectation_ineq}. 
		
		\begin{proof}[Proof of Lemma~\ref{lem:expectation_is_Theta_expectation_Ta}]
			Consider $a = \lfloor \expectation[X \mid X \le n-k]\,\rfloor$ from Lemmas~\ref{lem:expectation_is_Theta_expectation_Ta} and~\ref{lem:theta_bound_on_expectation_T}.
			Lemma~\ref{lem:binom_expectation_ineq} implies that $a \geq n/2 - \sqrt{\frac{n}{2\pi}} - o(\sqrt{n})$. Therefore, the $u$ from Lemma~\ref{lem:theta_bound_on_expectation_T} satisfies $u = \Theta(1)$. Consequently, Lemma~\ref{lem:theta_bound_on_expectation_T} implies that
			\[
			\expectation[T] = \Theta(\expectation[T(a)]).\qedhere
			\]
		\end{proof}

		From Lemma~\ref{lem:expectation_is_Theta_expectation_Ta}, we easily obtain the main result of this subsection, the following runtime estimate for the case $k=n/2 - \eps n$, which is tight apart from constant factors.
		
		\begin{theorem}
			\label{thm:linear_k}
			Let $k = n/2 - \eps n$, where $0 < \eps < 1/2$. Let $T$ be the time spent to optimize $\needle_{k}$. Then 
			\[
			\expectation[T] = \Theta\left( 2^n \binom{n}{k}^{-1} \right).
			\]
			This estimate is exponential in $n$ for all values of $\varepsilon$.
		\end{theorem}
		
		\begin{proof}
			Let $a = \lfloor \expectation[X \mid X \le n-k] \rfloor$. By Lemma~\ref{lem:expectation_is_Theta_expectation_Ta}, we have $\expectation[T] = \Theta(\expectation[T(a)])$. Theorem~\ref{thm:exact_expectation} gives
			\[
			\expectation[T(a)] = \sum_{j=a}^{n-k-1} \binom{n}{\leq j} \Big/ \binom{n-1}{j}.
			\]
			Also, Lemma~\ref{lem:binom_expectation_ineq} implies $a \geq n/2 - \sqrt{\frac{n}{2\pi}} - o(\sqrt{n})$, and so $j \geq a$ implies that $\binom{n}{\leq j} = \Theta(2^n)$. Consequently,
			\[
			\expectation[T(a)] = \Theta(2^n) \sum_{j=a}^{n-k-1} \binom{n-1}{j}^{-1},
			\]
			and the following Lemma~\ref{lem:recip_binom_coeff_eps_n} gives the main claim. That $2^n \binom{n}{k}^{-1} = \exp(\Theta(n))$ follows from estimating the binomial coefficient via Stirling's approximation $n! = \Theta( (\frac ne)^n \sqrt{n})$.
		\end{proof}
		
		\begin{lemma}
			\label{lem:recip_binom_coeff_eps_n}
			Let $k = n/2 - \eps n$, where $0 < \eps < 1/2$. Let $k \le g(n) \leq n-k-1$ and let $\displaystyle{S = \sum_{j = g(n)}^{n - k - 1} \binom{n-1}{j}^{-1}}$. Then $\displaystyle{S = \Theta\left(\binom{n}{k}^{-1}\right)}$.
		\end{lemma}
		
		\begin{proof}
			For the lower bound, we note that $S \ge \binom{n-1}{n-k-1}^{-1} = \binom{n-1}{k}^{-1} \ge \binom{n}{k}^{-1}$.
			
			The proof of the upper bound boils down to the fact that the first few terms (say $\eps n /2$ terms) of $\displaystyle{\binom{n-1}{k}^{-1}, \binom{n-1}{k+1}^{-1}, \binom{n-1}{k+2}^{-1},\ldots}$ decrease exponentially fast. Writing $M:=\lfloor \frac{n-1}{2} \rfloor$,  we first note that 
			\[
			S \le \sum_{j = k}^{n - k - 1} \binom{n-1}{j}^{-1} \le 2 \sum_{j = n - M -1}^{n - k - 1} \binom{n-1}{j}^{-1},
			\]
			hence it suffices to regard
			\[
			S' := \sum_{j = n - M - 1}^{n - k - 1} \binom{n-1}{j}^{-1}.
			\]
			Again from the symmetry of the binomial coefficients, we obtain
			\[
			\begin{aligned}
			S' &= \binom{n-1}{n - M - 1}^{-1} + \binom{n-1}{n - M}^{-1} + \binom{n-1}{n - M + 1}^{-1} + \cdots + \binom{n-1}{n - k - 1}^{-1} \\
			&= \binom{n-1}{k}^{-1} + \binom{n-1}{k+1}^{-1} + \binom{n-1}{k+2}^{-1} + \cdots + \binom{n-1}{M}^{-1}.
			\end{aligned}
			\]
			Because $k = n/2 - \eps n$, the first $\eps n/2$ terms in this sum decrease exponentially fast by some factor $\lambda \in (0, 1)$. Let $a = \binom{n-1}{k}^{-1}$. We have,
			\[
			\begin{aligned}
			S' &\leq \sum_{m=0}^{\eps n/2} a \lambda^m + \sum_{j = \eps n/2+1}^{n/2} \binom{n-1}{j}^{-1}\\
			&\leq a\frac{1}{1 - \lambda} + \lambda^{\eps n/2} a \cdot n/2\\
			&= O(a).
			\end{aligned}
			\]
			Noting that $\binom{n-1}{k} = \Theta(\binom{n}{k})$ finishes the proof.
		\end{proof}
		
		\subsection{$k$ Close to $n/2$}
		
		We now turn to the case that $n/2 - k = o(n)$. With $2^{(1-o(1))n}$ global optima, this case is not overly interesting, so we mainly present it to correct an incorrect statement in~\cite{DoerrK23majority}.
		
		\begin{theorem}
			\label{thm:bounds_for_k_near_half_n}
			Let $T$ be the runtime on $\needle_{k}$. 
			
			If $k = n/2 - g(n)$, where $g(n) = o(n)$ and $g(n) = \omega(\sqrt{n})$, then 
			\[
			\expectation[T] = O\left(g(n)2^n\binom{n}{k}^{-1}\right) \mbox{ and }
			\expectation[T] = \Omega\left(2^n\binom{n}{k}^{-1}\right).
			\]
			
			If $k = n/2 - O(\sqrt{n})$, then ${\expectation[T] = \Theta(n) = \Theta\left(\sqrt{n}2^n \binom{n}{k}^{-1}\right)}$.   
			
			All bounds in this theorem are sub-exponential in $n$.
		\end{theorem}
		
		Before proving this result, the incorrect comment in \cite{DoerrK23majority} that we noticed was made after Theorem~7 in that paper. Namely, it said that if $r=\Theta(1)$, which in our notation is $k = n/2 - \Theta(1)$, then the expected runtime is constant. This claim can easily be seen to be false because with constant probability the initial random solution is at least $\sqrt{n}$ away from the target region, and then you would need at least that long to reach the target. We observe that Theorem~\ref{thm:bounds_for_k_near_half_n} covers the case  $k = n/2 - \Theta(1)$ and shows a tight bound of $\Theta(n)$ for the expected runtime in this case.
		
		We note that Theorem~\ref{thm:bounds_for_k_near_half_n} does not give tight bounds in the case that $n/2-k \in \omega(\sqrt n) \cap o(n)$. We are confident that our analysis can be made more tight without much additional work, but we found the case $n/2-k = o(n)$ not interesting enough to justify the effort.
		
		\begin{proof}[Proof of Theorem~\ref{thm:bounds_for_k_near_half_n}]
			Let $a$ be as in Lemma~\ref{lem:theta_bound_on_expectation_T}. By Lemma~\ref{lem:expectation_is_Theta_expectation_Ta}, we have $\expectation[T] = \Theta(\expectation[T(a)])$.
			First, assume $k = n/2 - O(\sqrt{n})$. By Theorem~\ref{thm:exact_expectation}, we have
			\[
			\expectation[T(a)] = \sum_{j = a}^{n-k-1} \binom{n}{\leq j} \Big/ \binom{n-1}{j}.
			\]
			For all $j \geq a$ we have $\binom{n}{\leq j} = \Theta(2^n)$, and for all $j$ in the sum, $\binom{n-1}{j} = \Theta(2^n/\sqrt{n})$. Therefore, $\expectation[T] = \Theta(2^n)\Theta(\sqrt{n}2^{-n})(n/2 +O(\sqrt{n}) - a) = \Theta(\sqrt{n})\Theta(\sqrt{n}) = \Theta(n)$.
			
			Next, assume $k = n/2 - g(n)$.
			\[
			\expectation[T(a)] = \sum_{j = a}^{n/2 + g(n)} \binom{n}{\leq j} \Big/ \binom{n-1}{j}.
			\]
			As before, for all $j \geq a$ we have $\binom{n}{\leq j} = \Theta(2^n)$. Therefore, $T(a) \leq 2^n \sum_{j=a}^{n/2 + g(n)} \binom{n-1}{j}^{-1} = O\left(g(n)2^n\binom{n-1}{k}^{-1}\right) = O\left(g(n)2^n\binom{n}{k}^{-1}\right)$.
			Next, for the lower bound, $\expectation[T(a)] \geq \binom{n}{\leq n/2 + g(n)}\binom{n-1}{k}^{-1} = \Omega\left(2^n \binom{n}{k}^{-1}\right)$.
			
			From Stirling's approximation, we obtain that $\binom nk = 2^{(1-o(1)) n}$ for $k = n/2 - o(n)$, which shows that all bounds in this result are $\exp(o(n))$. 
		\end{proof}
		
		\subsection{$k$ Larger than $n/2$}
		
		For reasons of completeness, we also regard the case that $k$ is larger than $n/2$. We focus on a range in which the expected runtime is $o(1)$. While we are sure that similar arguments as found in this paper could be used to give precise bounds for other ranges of $k$, we did not do so because it did not seem important enough to be worth the effort. We merely note in passing that we think that for $k = n/2 + O(\sqrt{n})$ 
		the expected runtime is $\Theta(n)$.
		
		\begin{theorem}
			\label{thm:large_enough_k}
			Let $k \ge n/2 + \sqrt{n \log n}$, and let $T$ be the runtime on $\needle_{k}$. Then
			\[
			\expectation[T] = o(1).
			\]
		\end{theorem}
		\begin{proof}
			We use the upper bound in Lemma~\ref{lem:theta_bound_on_expectation_T}. 
			As in that lemma, let $X$ be a binomial random variable with parameters $n$ and $1/2$, and let $w = \Pr[X \le n-k]$. By the additive Chernoff bound, see, e.g., Theorem~1.10.7 in \cite{Doerr20bookchapter}, with $\lambda = \sqrt{n \ln n}$, we have
			\[
			w 
			\le \Pr\left[X \leq \expectation[X] - \lambda \right] 
			\leq  \exp\left(\frac{-2 \lambda^2}{n}\right) 
			= \frac 1 {n^2}.
			\]
			Let $a = \expectation[X \mid X \le n-k]$ be as in Lemma~\ref{lem:theta_bound_on_expectation_T}. Note that by definition of~$a$, we have $a \ge 0$. Hence by definition of $T(\cdot)$, we have $\expectation[T(a)] \le \expectation[T(0)]$, and thus, by Lemma~\ref{lem:theta_bound_on_expectation_T},
			\[
			\expectation[T] \leq w \expectation[T(a)] \le w \expectation[T(0)] \le \frac 1 {n^2} O(n^{3/2}) = o(1),
			\]
			where the estimate for $\expectation[T(0)]$ is from Lemma~\ref{lem:t0}, formulated below as a separate statement to keep this proof concise. 
		\end{proof}
		
		\begin{lemma}\label{lem:t0}
			Let $k \ge n/2$. Then $\expectation[T(0)] = O(n^{3/2})$.
		\end{lemma}
		
		\begin{proof}
			By Theorem~\ref{thm:exact_expectation}, we have
			\[
			\expectation[T(0)] = \sum_{j=0}^{n-k-1} \binom{n}{\le j} \Big/ \binom{n-1}{j},
			\]
			where recall $\binom{n}{\le j} / \binom{n-1}{j} = \expectation[T_j^+]$ by Lemma~\ref{lem:droste_simplified}. 
			By Lemma~\ref{lem:Tip_is_increasing}, $\expectation[T_j^+]$ is increasing for $j \in [0..n]$. We estimate $\expectation[T_j^+] \le \expectation[T_{n/2-1}^+]$, irrespective of whether the state $n/2-1$ belongs to the global optimum or not, and compute 
			\[
			\expectation[T_{n/2-1}^+] 
			= \binom{n}{\le n/2-1} \Big/ \binom{n-1}{n/2-1} 
			\le 2^{n-1} \Big/ (2^{n-1} \Theta(n^{-1/2}))
			= \Theta(n^{1/2}).
			\]
			Consequently, $\expectation[T(0)] \le (n-k) \expectation[T_{n/2-1}^+] = O(n^{3/2})$.
		\end{proof}
		
		\section{Conclusion}
		
		In this work, we determined the precise expected runtime of the \emph{randomized local search} heuristic on the generalized \needle problem and gave easy-to-handle estimates for this runtime. Our work gives the first lower bounds for this problem and improves the upper bounds shown recently in~\cite{DoerrK23majority}. The key to these results is an elementary Markov chain approach, which greatly simplifies the complex drift analysis approach of~\cite{DoerrK23majority}. From the differences in the two works, we would conclude that using drift analysis in situations where there is no natural drift (as on the plateau of constant fitness of the \needle problem) is not an ideal approach. As long as other methods are available, we would recommend to reserve the use of drift analysis to proving upper bounds when there is a natural drift towards the target, or to proving lower bounds when there is drift away from the target.

		\bibliographystyle{alphaurl}

	}
\end{document}